\documentclass[final]{article}

     \usepackage[nonatbib]{neurips_2020}

\usepackage[utf8]{inputenc} 
\usepackage[T1]{fontenc}    
\usepackage{hyperref}       
\usepackage{courier}
\usepackage{url}            
\usepackage{booktabs}       
\usepackage{amsfonts}       
\usepackage{nicefrac}       
\usepackage{microtype}      
\usepackage{enumitem}
\usepackage{wrapfig}

\usepackage{algorithm}
\usepackage{algorithmic}

\usepackage{times}
\usepackage{epsfig}
\usepackage{graphicx}
\usepackage{amsmath}
\usepackage{amssymb}
\usepackage{ mathrsfs }
\usepackage{mathtools}
\usepackage{xcolor}
\usepackage{xspace}
\usepackage{amsthm}
\usepackage{subcaption}

\DeclareMathOperator*{\argmin}{arg\,min}

\newtheorem{definition}{Definition}
\newtheorem{theorem}{Theorem}
\newtheorem*{theorem*}{Theorem}
\newtheorem{lemma}{Lemma}
\newtheorem*{lemma*}{Lemma}
\usepackage{hyperref}
\usepackage{xcolor}

\title{Neurosymbolic Reinforcement Learning with Formally Verified Exploration}

\author{%
 Greg Anderson \\
 UT Austin \\
  \texttt{\small ganderso@cs.utexas.edu} \\
\And
Abhinav Verma \\
UT Austin \\
\texttt{\small verma@utexas.edu}
\And
Isil Dillig \\
UT Austin \\
\texttt{\small isil@cs.utexas.edu}
\And
Swarat Chaudhuri \\
UT Austin \\
\texttt{\small swarat@cs.utexas.edu}
}

\begin{document}

\maketitle

\newcommand{\toolname}{\textsc{Revel}\xspace}
\newcommand{\propel}{\textsc{Propel}\xspace}
\newcommand{\todo}[1]{{\color{red}#1}}
\newcommand{\G}{\mathcal{G}}

\newcommand{\cpo}{\textsc{Cpo}\xspace}
\newcommand{\ddpg}{\textsc{Ddpg}\xspace}
\newcommand{\daggeralgo}{\textsc{Dagger}\xspace}

\definecolor{mygreen}{rgb}{0,0.4,0}
\newcommand{\comment}[1]{\textit{\textcolor{mygreen}{//#1}}}

\newcommand{\states}{\mathcal{S}}
\newcommand{\actions}{\mathcal{A}}
\newcommand{\supp}{\operatorname{supp}}
\newcommand{\reach}{\operatorname{reach}}
\newcommand{\proj}{\operatorname{proj}}

\newcommand{\rl}{RL\xspace}

\newcommand{\FF}{\mathcal{F}}
\newcommand{\HH}{\mathcal{H}}
\newcommand{\NN}{\mathcal{N}}
\newcommand{\ifc}{\mathbf{if}}
\newcommand{\thenc}{\mathbf{then}}
\newcommand{\elsec}{\mathbf{else}}

\newcommand{\method}{\textsc{Snapper}\xspace}
\newcommand{\F}{\mathcal{F}}

\newcommand{\update}{\textsc{Update}}
\newcommand{\lift}{\textsc{Lift}}
\newcommand{\project}{\textsc{Project}}

\begin{abstract}
We present \toolname, a partially neural reinforcement learning (\rl) framework for provably  safe exploration in continuous state and action spaces. A key challenge for provably safe deep RL is that repeatedly verifying neural networks within a learning loop is computationally infeasible. We address this challenge using two policy classes: a general, neurosymbolic class with approximate gradients and a more restricted class of symbolic policies that allows efficient verification. Our learning algorithm is a mirror descent over policies: in each iteration, it safely lifts a symbolic policy into the neurosymbolic space, performs safe gradient updates to the resulting policy, and projects the updated policy into the safe symbolic subset, all without requiring explicit verification of neural networks. 
Our empirical results show that \toolname enforces safe exploration in many scenarios in which Constrained Policy Optimization does not, and that it can discover policies that outperform those learned through prior approaches to verified exploration. 
\end{abstract}

\section{Introduction}
\label{sec:Intro}

Guaranteeing that an agent behaves {safely} during exploration is a fundamental problem in reinforcement learning (\rl)~\cite{garcia2015comprehensive,cpo_main}. Most approaches to the problem are based on stochastic definitions of safety~\cite{MoldovanA12,chow2018lyapunov,cpo_main,chow2018lyapunov}, requiring the agent to satisfy a safety constraint with high probability or in expectation. However, in applications such as autonomous robotics, unsafe agent actions --- no matter how improbable --- can lead to cascading failures with high human and financial costs. As a result, it can be important to ensure that the agent behaves safely {\em even on worst-case inputs}.

A number of recent efforts~\cite{AlshiekhBEKNT18,fulton2018safe} use {formal methods} to offer such worst-case guarantees during exploration. Broadly, these methods 
construct a space of provably safe policies \emph{before} the learning process starts. 
Then, during exploration, a {\em safety monitor} observes the learner, 
forbidding all actions that cannot result from one of these safe policies. If the learner is about to take a forbidden action, a safe policy (a {\em safety shield}) is executed instead.

So far, these methods have only been used to discover policies over simple, finite action spaces. Using them in more complex settings --- in particular, continuous action spaces --- is much more challenging. A key problem is that these safety monitors  are constructed a priori and are blind to the internal state of the learner. As we experimentally show later in this paper, such a ``one-size-fits-all'' strategy can unnecessarily limit exploration and impede learner performance. 


In this paper, we improve this state of the art through an \rl framework, called \toolname \footnote{\toolname stands for {\bf Re}inforcement learning with {\bf ve}rified exp{\bf l}oration. The current implementation is available at \url{https://github.com/gavlegoat/safe-learning}.} that allows learning over continuous state and action spaces, supports (partially) neural policy representations and contemporary policy gradient methods for learning, while also ensuring that every intermediate policy that the learner constructs during exploration is safe on worst-case inputs. Like previous efforts, \toolname uses monitoring and shielding. However, unlike in prior work, the monitor and the shield are updated 
as learning progresses.   




A key feature of our approach is that we repeatedly invoke a formal verifier from within the learning loop. Doing this is challenging because of the high computational cost of verifying neural networks. 
We overcome this challenge using a neurosymbolic policy representation in which the shield and the  monitor are expressed in an easily-verifiable symbolic form, whereas the normal-mode policy is given by a neural network. 
Overall, this representation admits efficient gradient-based learning as well as efficient updates to both the shield and monitor. 

To learn such neurosymbolic policies,  we build on \propel~\cite{propel}, a recent \rl framework in which policies are represented in compact symbolic forms (albeit without consideration of safety), and design a learning algorithm that performs a functional mirror descent in the space of  neurosymbolic policies. The algorithm views the set of shields as being obtained by imposing a constraint on the general policy space. 
Starting with a safe but suboptimal shield, it alternates between: (i) safely {\em lifting} the current shield into the unconstrained policy space by adding a neural component; (ii) safely {\em updating} this neurosymbolic policy using approximate gradients; and (iii) using a form of imitation learning 
to {\em project} the updated policy back into the constrained space of shields. Importantly, none of these steps requires direct verification of neural networks.

Our empirical evaluation, on a suite of continuous control problems, shows that \toolname enforces safe exploration in many scenarios where established \rl algorithms (including \cpo~\cite{cpo_main}, which is motivated by safe \rl) do not, 
while discovering policies that 
outperform policies based on static shields. 
Also, building on results for \propel, we develop a theoretical analysis of \toolname.

In summary, the contributions of the paper are threefold. First, we introduce the first \rl approach to use deep policy representations and policy gradient methods while guaranteeing formally verified exploration. Second, we propose a  new solution to this problem that combines ideas from \rl and formal methods, 
and we show that our method has convergence guarantees. Third, we present promising experimental results for our method in the continuous control setting. 

\renewcommand{\L}{\mathcal{L}}
\newcommand{\Safe}{\mathit{Safe}}

\section{Preliminaries}
\label{sec:Problem}


\paragraph{Safe Exploration.} We formulate our problem in terms of a Markov Decision Process (MDP) that has the standard 
probabilistic dynamics, as well as a 
worst-case dynamics that is used for verification. Formally, such an MDP is a structure $M = (\states, \actions, P, c, \gamma, p_0, P^\#, \states_U)$. Here, $\states$ is a set of environment states; $\actions$ is a set of agent actions; 
$P(s' \mid s, a)$ is a probabilistic transition function;
$c : \states \times \actions \to \mathbb{R}$ is a state-action cost function; 
$0 < \gamma < 1$ is a discount factor; 
$p_0(s)$ is an initial state distribution with support $\states_0$; 
$P^\#: \states \times \actions \rightarrow 2^\states$ is a deterministic function that defines {\em worst-case bounds} on the environment behavior; and $\states_U$ is a designated set of {\em unsafe states} that the learner must always avoid.
Because our focus is on continuous domains, we assume that $\states$ and $\actions$ are real vector spaces. The function $P^\#$ is assumed to be available in closed form to the learner; because it  captures worst-case dynamics, we require that $\supp(P(s' \mid s, a)) \subseteq P^\#(s, a)$ for all $s, a$. In general, the method for obtaining $P^\#$ will depend on the problem under consideration. In Section~\ref{sec:experiments} we explain how we generate these worst case bounds for our experimental environments.

A \emph{policy} for $M$ is a (stochastic) map $\pi : \states \to \actions$ that determines which action the agent should take in a given state. Each policy $\pi$ induces a probability distribution on the cost $c_i$ at each time step $i$. 
The aggregate cost of a policy $\pi$ is $J(\pi) = \mathrm{E}[\sum_i \gamma^i c_i]$, where $c_i$ is the cost at the $i$-th time step. 

For a set $S \subseteq \states$, we define the set of states $\reach_i(\pi, S)$ that are reachable from $S$ in $i$ steps under worst-case dynamics:

\vspace{-0.5cm}
$$
\begin{array}{l}
\reach_1(\pi, S) =  \bigcup_{s \in S, a \in \supp(\pi( \cdot \mid s))} P^\#(s, a) \qquad \qquad 
\reach_{i+1}(\pi, S) = \reach_1(\pi, \reach_{i}(\pi, S)).
\end{array}
$$
The policy $\pi$ is \emph{safe} if $(\bigcup_i \reach_i(\pi, \states_0)) \cap \states_U = \varnothing$. If $\pi$ is safe, we write $\Safe(\pi)$.


We define a {\em learning process} as a sequence of policies $\L = \pi_0,\pi_1,\dots, \pi_m$. We assume that the initial policy $\pi_0$ in this sequence is worst-case safe. Our algorithmic objective is to discover a learning process $\L$ such that the final policy $\pi_m$ is safe and optimal, and every intermediate policy is safe:
\begin{equation}
\pi_m = \argmin_{\pi \textrm{~~s.t.~~} \Safe(\pi)} J(\pi) 
\label{eq:opt}
\end{equation}
\begin{equation}
\forall 0 \le i \le m: \Safe(\pi_i).
\end{equation}

\paragraph{Formal Verification.} 
Our learning algorithm relies on an oracle for formal verification of policies. Given a policy $\pi$, such a verifier tries to construct an inductive proof of the property $\Safe(\pi)$. Such a proof takes the form of an {\em inductive invariant}, defined as a set of states $\phi$ such that: (i) $\phi$ includes the initial states, i.e., $\states_0 \subseteq \phi$; (ii) $\phi$ is closed under the worst-case transition relation, i.e., $\reach_1(\phi) \subseteq \phi$; and (iii) $\phi$ does not overlap with the unsafe states, i.e., $\phi \cap \states_U = \emptyset$. Intuitively, states $s$ in $\phi$ are such that even under worst-case dynamics, MDP trajectories from $s$ can never encounter an unsafe state. 
We use the notation  $\phi \vdash \pi$ to indicate that policy $\pi$ can be proven safe using inductive invariant $\phi$.

Inductive invariants can be constructed in many ways. Our implementation uses {\em abstract interpretation}~\cite{cousot1977}, which maintains some \emph{abstract} state that approximates the \emph{concrete} states which the system can reach. For example, the abstract state might be a hyperinterval in the state space of the program that defines independent bounds on each state variable. Critically, this abstract state is an \emph{overapproximation}, meaning that while the abstract state may include states which are not actually reachable, it will always include \emph{at least} every reachable state. By starting with an abstraction of the initial states and using abstract interpretation to propagate this abstract state through the environment transitions and the policy, we can obtain an abstract state which includes all of the reachable states of the system (that is, we compute approximations of $\reach_i(\states_0)$ for increasing $i$). Then if this  abstract state does not include any unsafe states, we can be sure that none of the unsafe states are reachable by any concrete trajectory of the system either.
\section{Learning Algorithm}
\label{sec:algorithm}


\begin{algorithm}[t]
	\caption{Reinforcement Learning with Formally Verified Exploration (\toolname)}
	\label{alg:algo}
	\begin{small}
	\begin{algorithmic}[1]
		\STATE  {\bfseries Input:} Symbolic Policy Class $\G$ \& Neural Policy Class $\F$. 
		\STATE  {\bfseries Input:} Initial $g_0 \in \G$, with the guarantee $\phi_0 \vdash \Safe(g_0)$ for some $\phi_0$ 
		\STATE Define neurosymbolic policy class $\HH = \{ h(s) \equiv \ifc~P^\#(s, f(s)) \subseteq \phi~\thenc~f(s) ~\elsec~g(s)\}$
		\FOR{$t = 1, \ldots, T$}
    		\STATE $h_{t} \leftarrow \lift_\HH(g_t, \phi_t)$ \ \ \ \ \ \ \ \ \ \comment{lifting the new symbolic policy and proof into the blended space}
		    \STATE $h_t \leftarrow \update_\F(h_t, \eta)$ \ \ \ \ \ \ \ \ \ \comment{policy gradient in neural policy space with learning rate $\eta$}
		    \STATE $(g_{t + 1}, \phi_{t + 1}) \leftarrow \project_\Pi(h_t)$ \ \ \ \ \comment{synthesis of safe symbolic policy and corresponding invariant}
		\ENDFOR
		\STATE \textbf{Return:} Policy $h_T$
	\end{algorithmic}
	\end{small}
\end{algorithm}

Our learning method is a functional mirror descent in policy space, based on approximate gradients, similar to \propel~\cite{propel}. The algorithm relies on two policy classes $\G$ and $\HH$, with $\G \subseteq \HH$.

The class $\G$ comprises the policies that we use as shields. These policies are safe and can be efficiently certified as such. Because automatic verification works better on functions that belong to certain restricted classes and are represented in compact, symbolic forms, we assume some syntactic restrictions on our shields. The specific restrictions depend on the power of the verification oracle; we describe the choices made in our implementation in Section~\ref{sec:piecewise-linear}.

The larger class $\HH$ consists of neurosymbolic policies. Let $\F$ be a predefined class of neural policies. We assume that each shield in $\G$ can also be expressed as a policy in $\F$, i.e., $\G \subseteq \F$. Policies $h \in \HH$ are of the form: 

\vspace{-0.5cm}
$$
h(s) = \ifc~(P^\#(s, f(s)) \subseteq \phi)~\thenc~f(s)~\elsec~g(s)
$$

where $g \in \G$, $f \in \F$, and $\phi$ is an inductive invariant that establishes $\Safe(g)$. We commonly denote a policy $h$ as above by the notation $(g, \phi, f)$. 
    
The ``true'' branch in the definition of $h$ represents the normal mode of the policy. The condition  $P^\#(s, f(s)) \subseteq \phi$ is the safety monitor. If this condition holds, then the action $f(s)$ is safe, as it can only lead to states in $\phi$ (which does not overlap with the unsafe states). If the condition does not hold, then $f$ can violate safety, and the shield $g$ is executed in its place. In either case, $h$ is safe. As for updates to $h$, we do not assume that the policy gradient $\nabla_\HH J(h)$ in the space $\HH$ exists, and approximate it by the gradient $\nabla_\F J(h)$ in the space $\F$ of neural policies.

We sketch our learning procedure in Algorithm~\ref{alg:algo}. The algorithm starts with a (manually constructed) shield $g_0 \in \G$ and a corresponding invariant $\phi_0$, then iteratively performs the following steps.

\textbf{$\lift_\HH$.~}
This step takes as input a shield $g \in \G$ and its accompanying invariant $\phi$, and constructs the policy $(g, \phi, g) \in \HH$. Note that the neural component of this policy is just the input shield $g$ (in a neural representation). In practice, to construct this component, we can train a randomly initialized neural network to imitate $g$, using an algorithm such as \daggeralgo~\cite{dagger}. Because the safety of any policy $(g, \phi, f')$ only depends on $g$ and $\phi$, this step is safe. 


\textbf{$\update_\F$.}
This procedure performs a series of gradient updates to a neurosymbolic policy $h = (g, \phi, f)$. As mentioned earlier, this step uses the approximate policy gradient $\nabla_\F J(h)$. 
This means that after an update, the new policy is $(g, \phi, f - \eta \nabla_\F J(h))$, for a suitable learning rate $\eta$. As the update does not change $g$ and $\phi$, the new policy is provably safe. Also, we show later that, under certain assumptions,
the regret introduced by our approximation of the gradient is bounded.

\textbf{$\project_\G$.~}
This procedure implements the projection operation of mirror descent. 
Given a neurosymbolic policy $h = (g,\phi, f)$, the procedure computes a policy $g' \in \G$ that satisfies
$g' = \argmin_{g'' \in \G} D(g'', (g, \phi, f))$
for some Bregman divergence $D$. Along with $g'$, we compute an invariant $\phi'$ such that $\phi' \vdash \Safe(g')$.

The computation of $g'$ can be naturally cast as an imitation learning task with respect to the  demonstration oracle $(g, \phi, f)$. Prior work~\cite{propel,pirl} has given heuristic solutions to this problem for the case when $g''$ obeys certain syntactic constraints. In our setting, we have an additional semantic requirement: $g''$ must be provably safe. 
How to solve this problem depends on the precise definition of the class of shields $\G$. Section~\ref{sec:piecewise-linear} sketches the approach to this problem used in our implementation.


\subsection{Instantiation with Piecewise Linear Shields}\label{sec:piecewise-linear}

\begin{algorithm}[t]
	\caption{Implementation of $\project_\G$}
	\label{alg:project}
	\begin{small}
	\begin{algorithmic}[1]
		\STATE  {\bfseries Input:}  A neurosymbolic policy $h = (g, \phi, f)$ where $g = [(g_1, \chi_1), \ldots, (g_n, \chi_n)]$
		\STATE $g^* \gets g$
		\vspace{0.05in}
		\FOR{$t = 1, \ldots, T$}
		    \STATE $\psi \gets {\textsc{CuttingPlane}}(\chi_i)$ \textrm{ for heuristically selected $i$}  
		    \STATE $g_i^1 \gets \textsc{ImitateSafely}(f, g_i, \ \chi_i \land \psi)$; \qquad
		     $g_i^2 \gets \textsc{ImitateSafely}(f, g_i, \ \chi_i \land \neg \psi)$
		    \STATE $g' \gets \textsc{Split}(g, i, (g_i^1, \chi_i \land \psi), (g_i^2, \chi_i \land \neg \psi))$   \IF{$D(g', h) < D(g^*, h)$}
     \STATE $g^* \gets g'$
   \ENDIF
		\ENDFOR
		\STATE $\phi^* \gets \textsc{SafeSpace}(g^*)$
   \STATE \textbf{return} $(g^*, \phi^*)$
	\end{algorithmic}
	\end{small}
\end{algorithm}

Any attempt to implement \toolname must start by choosing a class $\G$ of shields. 
Policies in $\G$ should be sufficiently expressive to allow for good learning performance 
but also facilitate verification. 
In our implementation, we choose $\G$ to comprise \emph{deterministic, piecewise linear policies} 
of the form:

\vspace{-0.4cm}
    $$
    g(s) = \begin{cases}
    g_1(s) & \textrm{if}~\chi_1(s) \\
    g_2(s) & \textrm{if}~\chi_2(s) \wedge \neg \chi_1(s)\\
    \dots \\
    g_n(s) & \textrm{if}~\chi_n(s) \wedge (\bigwedge_{1 \leq i <n} \neg \chi_i(s)),
    \end{cases}
    $$
    where $\chi_1,\dots, \chi_n$ are linear predicates that {\em partition} the state space, and each $g_i$ is a linear function. 
We represent $g(s)$ as a list of pairs $(g_i, \chi_i)$. We refer to the subpart of the state space defined by $\chi_i \land \bigwedge_{j=1}^{i-1} \neg \chi_j$ as the \emph{region} for linear policy $g_i$ and denote this region by $\mathsf{Region}(g_i)$. 

Now we sketch our implementation of Algorithm~\ref{alg:algo}.
Since the $\lift_\HH$ and $\update_\F$ procedures are agnostic to the choice of $\G$, we focus on $\project_\G$, which seeks to find a shield $g$ at minimum imitation distance $D(g, h)$ from a given $h \in \HH$. 


Our implementation of this operation is the iterative procedure in  Algorithm~\ref{alg:project}. Here, we start with an input policy $h = (g, \phi, f)$.
In each iteration, we identify a component $g_i$ with region $\chi_i$, then perform
the following steps:
(i) Sample a {\em cutting plane} that creates a more fine-grained partitioning of the safe region, by splitting the region $\chi_i$ into two new regions $\chi_i^1$ and $\chi_i^2$. 
(ii) For each new region $\chi_i^j$, use a subroutine {\sc ImitateSafely} to construct a  safe linear policy $g_i^j$ (and a corresponding invariant) that minimizes $D(g_i^j, h)$ within the region $\chi_i^j$.
(iii) Replace $(g_i, \chi_i)$ by the two new components, leading to the creation of a 
new, refined shield $g'$. 
The procedure ends by returning the most optimal shield $g'$ (and an invariant obtained by combining the invariants of the $g_i^j$-s) constructed through this process. 
Now we sketch {\sc ImitateSafely}, which constructs safe and imitation-loss-minimizing linear policies.
By collecting state-action pairs using \daggeralgo~\cite{dagger}, we reduce the procedure's objective to a series of constrained supervised learning problems. Each of these problems is solved 
using a 
projected gradient descent (PGD) that alternates between gradient updates to a linear policy and projections into the set of safe linear policies. Critically, the constraint imposed on each of these optimization problems is constructed such that (i) the resulting policy is provably safe and (ii) the projection for the PGD algorithm is easy to compute. In our implementation these constraints take the form of a hyperinterval in the parameter space of the linear policies. We can then use abstract interpretation~\cite{cousot1977}, a common framework for program verification, to prove that every controller within a particular hyperinterval behaves safely. For more details on {\sc ImitateSafely}, see the supplementary material.

\subsection{Theoretical Analysis}
\label{sec:Theory}

The \toolname approach introduces two new sources of error over standard mirror descent. First, we approximate the gradient $\nabla_\HH$ by $\nabla_\F$, which introduces bias. 
Second, our projection step may be inexact. Prior work~\cite{propel} has studied methods for implementing the projection step with bounded error. 
Here, we bound the bias in the gradient approximation under some simplifying assumptions, and use this result to prove a regret bound on the final shield that our method converges on.
We define a safety indicator $Z$ which is zero whenever the shield is invoked and one otherwise. We assume:
\begin{enumerate}[leftmargin=15pt]
    \item $\HH$ is a vector space equipped with an inner product $\langle \cdot , \cdot \rangle$ and induced norm $\|h\| = \sqrt{\langle h , h \rangle}$,
    \item $J$ is convex in $\HH$, and 
    $\nabla J$ is $L_J$-Lipschitz continuous on $\HH$,
    \item $\HH$ is bounded (i.e., $\sup \{\|h - h'\| \mid h,h' \in \HH\} < \infty$),
    \item $\mathbb{E}[1- Z] \le \zeta$, i.e., the probability that the shield is invoked is bounded above by $\zeta$,
    \item the bias introduced in the sampling process is bounded by $\beta$, i.e., $\|\mathbb{E}[\widehat{\nabla}_\F \mid h] - \nabla_\FF J(h) \| \le \beta$, where 
    $\widehat{\nabla}_\F$ is the estimated gradient 
    \item for $s \in \states$, $a \in \actions$, and policy $h \in \HH$, if $h(a \mid s) > 0$ then $h(a \mid s) > \delta$ for some fixed $\delta > 0$.
\end{enumerate}
Intuitively, this last assumption amounts to cutting of the tails of the distribution so that no action can be arbitrarily unlikely.
Now, let the variance of the gradient estimates be bounded by $\sigma^2$, and assume the projection error $\|g_t - g_t^*\| \le \epsilon$ where $g_t^*$ is the exact projection of a neurosymbolic policy onto $\G$ and $g_t$ is the computed projection. Let $R$ be an $\alpha$-strongly convex and $L_R$-strongly smooth regularizer. Then the bias of our gradient estimate is bounded by Lemma~\ref{lem:bias} and the expected regret bound is given by Theorem~\ref{thm:regret}.

\begin{lemma}
\label{lem:bias}
  Let $\gamma$ be the diameter of $\HH$, i.e., $\gamma = \sup\{\|h - h'\| \mid h,h' \in \HH\}$. Then the bias incurred by approximating $\nabla_\HH J(h)$ with $\nabla_\FF J(h)$ and sampling is bounded by
  \[\left\|\mathbb{E}\left[\widehat{\nabla}_t \mid h\right] - \nabla_\HH J(h)\right\| = 
  O(\beta + L_j \zeta).\]
\end{lemma}

\begin{theorem}
\label{thm:regret}
  Let $g_1, \ldots, g_T$ be a sequence of shields in $\G$ returned by \toolname{} and let $g^*$ be the optimal programmatic policy. Choosing a learning rate $\eta = \sqrt{\frac{1}{\sigma^2} \left( \frac{1}{T} + \epsilon\right)}$ we have the expected regret over $T$ iterations:
  \[\mathbb{E}\left[\frac{1}{T} \sum_{i=1}^T J(g_i)\right] - J(g^*) = O\left( \sigma \sqrt{\frac{1}{T} + \epsilon} + \beta + L_J \zeta \right)\]
\end{theorem}

This theorem matches the expectation that when a blended policy $h = (g, \phi, f)$ is allowed to take more actions without the shield intervening (i.e., $\zeta$ decreases), the regret bound is decreased. Intuitively, this is because when we use the shield, the action we take does not depend on the neural network $f$, so the learner does not learn anything useful. However if $h$ is using $f$ to choose actions, then we have unbiased gradient information as in standard RL.

\section{Experiments}
\label{sec:experiments}

\begin{figure}
\captionsetup[subfigure]{aboveskip=-1pt, belowskip=-2pt}
\centering
\begin{subfigure}{0.42\textwidth}
    \centering
    \includegraphics[width=\textwidth]{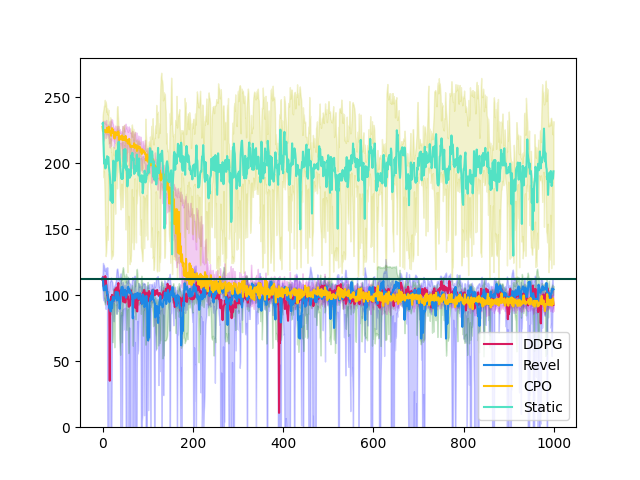}
    \caption{mountain-car}
    \label{fig:mountain_car_training}
\end{subfigure}
\begin{subfigure}{0.42\textwidth}
    \centering
    \includegraphics[width=\textwidth]{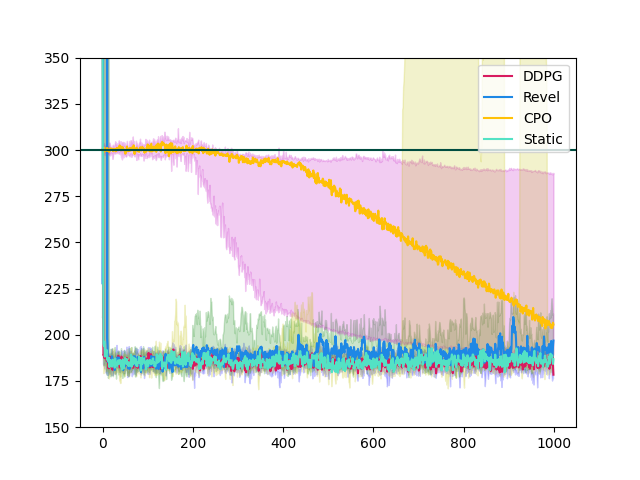}
    \caption{road}
    \label{fig:road_training}
\end{subfigure}
\begin{subfigure}{0.42\textwidth}
    \centering
    \includegraphics[width=\textwidth]{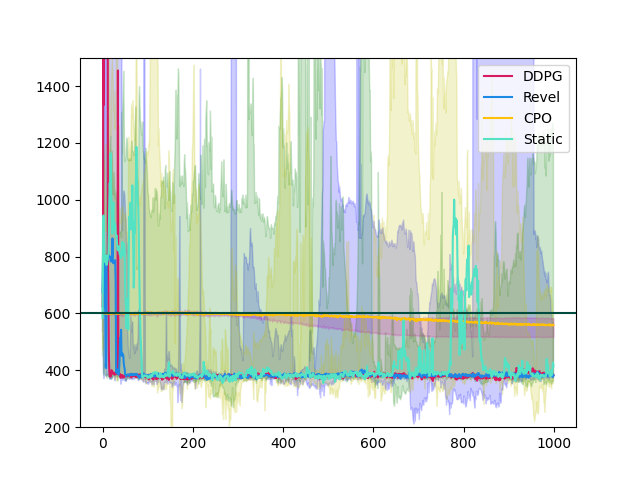}
    \caption{road-2d}
    \label{fig:road_2d_training}
\end{subfigure}
\begin{subfigure}{0.42\textwidth}
    \centering
    \includegraphics[width=\textwidth]{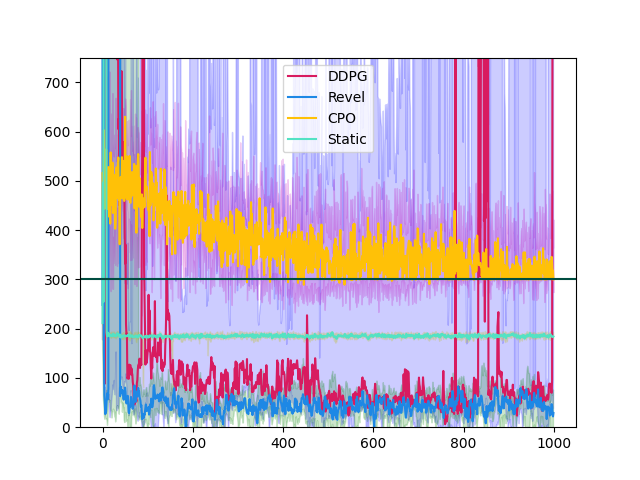}
    \caption{noisy-road}
    \label{fig:noisy_road_training}
\end{subfigure}
\begin{subfigure}{0.42\textwidth}
    \centering
    \includegraphics[width=\textwidth]{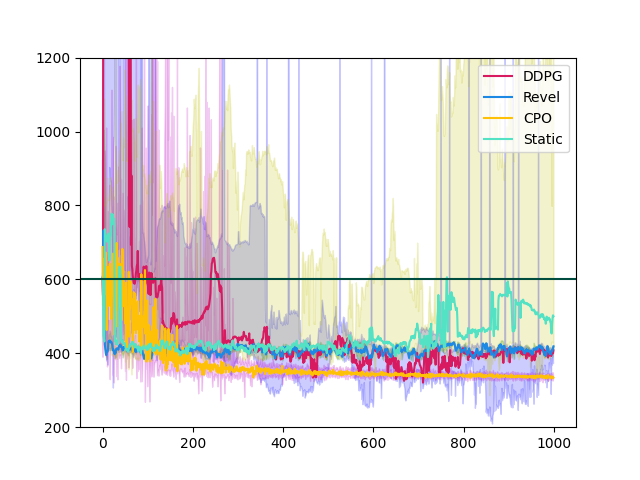}
    \caption{noisy-road-2d}
    \label{fig:noisy_road_2d_training}
\end{subfigure}
\begin{subfigure}{0.42\textwidth}
    \centering
    \includegraphics[width=\textwidth]{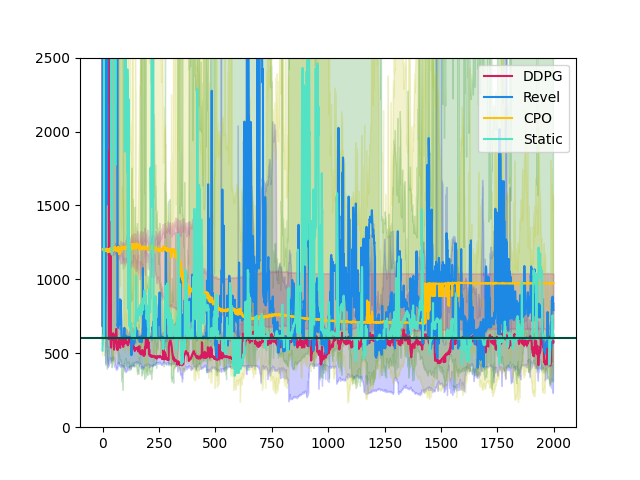}
    \caption{obstacle}
    \label{fig:obstacle_training}
\end{subfigure}
\begin{subfigure}{0.42\textwidth}
    \centering
    \includegraphics[width=\textwidth]{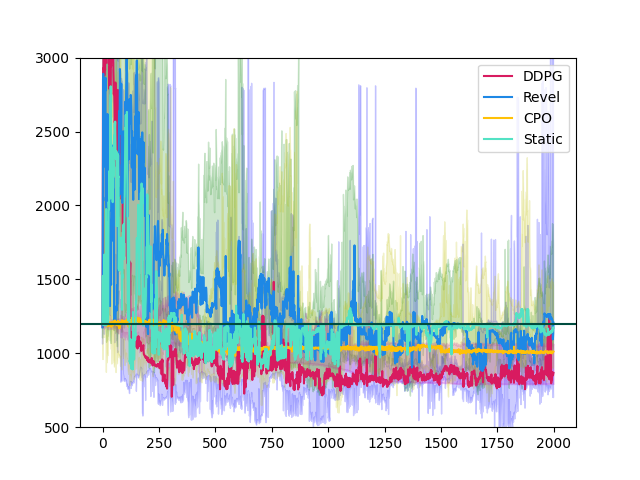}
    \caption{obstacle2}
    \label{fig:mid_obstacle_training}
\end{subfigure}
\begin{subfigure}{0.42\textwidth}
    \centering
    \includegraphics[width=\textwidth]{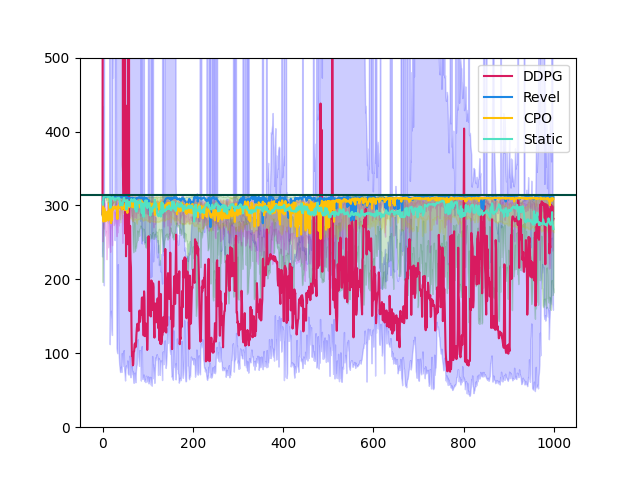}
    \caption{pendulum}
    \label{fig:pendulum_training}
\end{subfigure}
\caption{Training performance comparison on our benchmarks. The y-axis represents the Cost $J(\pi)$ and the x-axis gives the number of training episodes.}
\label{fig:curves}
\vspace{-0.25in}
\end{figure}

\begin{figure}
\captionsetup[subfigure]{aboveskip=-1pt, belowskip=-2pt}
\centering
\begin{subfigure}{0.42\textwidth}
    \centering
    \includegraphics[width=\textwidth]{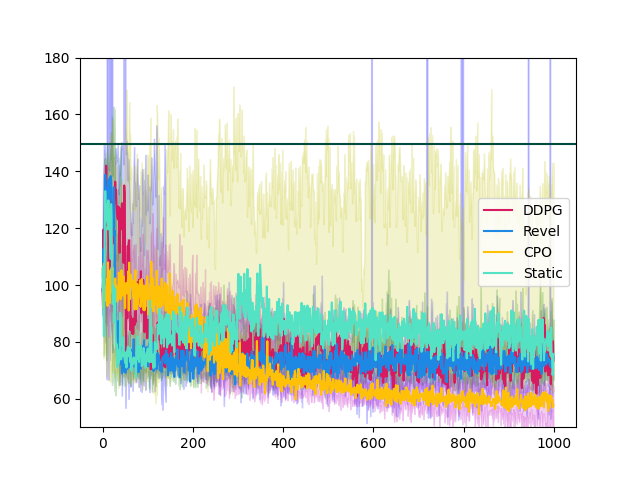}
    \caption{acc}
    \label{fig:acc_training}
\end{subfigure}
\begin{subfigure}{0.42\textwidth}
    \centering
    \includegraphics[width=\textwidth]{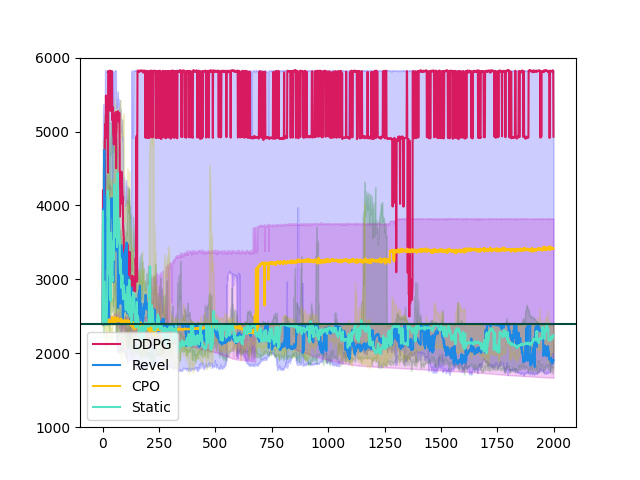}
    \caption{car-racing}
    \label{fig:car_racing_training}
\end{subfigure}
\caption{Training performance comparison (continued). The y-axis represents the Cost $J(\pi)$ and the x-axis gives the number of training episodes.}
\label{fig:curves-cont}
\vspace{-0.2in}
\end{figure}

Now we present our empirical evaluation of \toolname{}. 
We investigate two questions: 
(1) How much {safer} are \toolname{} policies compared to state-of-the-art RL techniques that lack worst-case safety guarantees? What is the   
performance penalty for this increased safety?
(2) Does \toolname offer significant performance gains over prior 
verified exploration approaches based on static shields\cite{fulton2018safe,AlshiekhBEKNT18}?

To answer these questions, we compared \toolname{} against three baselines: (1) Deep Deterministic Policy Gradients (\ddpg)~\cite{NN:DDPG};
(2) Constrained policy optimization (\cpo)~\cite{cpo_main};
and (3) a variant of \toolname{} that never updates the user-provided shield.
Of these, \cpo is designed for safe exploration and takes into account a safety cost function. For \ddpg, we engineered a reward function that has a penalty for safety violations. Details of hyperparameters that we used appear in the Appendix.

Our experiments used 10 benchmarks that include 
classic control problems, robotics applications, and benchmarks from prior work~\cite{fulton2018safe}. For each of these environments, we hand-constructed a worst-case, piecewise linear model of the dynamics. These models are based on the physics of the environment and use non-determinism to approximate nonlinear functions. For example, some of our benchmarks include trigonometric functions which cannot be represented linearly. In these cases, we define piecewise linear upper and lower bounds to the trigonometric functions. These linear approximations are necessary to make verification feasible. Each benchmark also includes a bounded-time safety property which should hold at all times during training.




\textbf{Performance.} First, we compare the policies learned using \toolname{} against policies learned using the baselines in terms of their cost (lower is better). 
Figures~\ref{fig:curves} and~\ref{fig:curves-cont} show the cost over time of the policies during training. The results suggest that:
\begin{itemize}[leftmargin=10pt]
    \item The performance of \toolname is competitive with (or even better than) DDPG for 7 out of the 10 benchmarks. \toolname achieves significantly better reward than DDPG in the ``car-racing'' benchmark, and reward is only slightly worse for 2 benchmarks. 
    \item  \toolname has better performance than CPO on 4 out of the 10 benchmarks and only performs slightly worse on 2. Furthermore, the cost incurred by CPO is significantly worse on 2 benchmarks (noisy-road and car-racing).
    \item \toolname outperforms the static shielding approach on 4 out of 10 benchmarks. Furthermore, the difference is very substantial on two of these benchmarks (noisy-road and mountain-car). 
\end{itemize}

\toolname{} does induce substantial overhead in terms of computational cost. The cost of the network updates and shield updates for each benchmark are shown in Table~\ref{tab:runtime} along with the percentage of the total time spent in shield synthesis. The ``acc'' and ``pendulum'' benchmarks stand out as having very fast shield updates. For these two benchmarks the safety properties are relatively simple, so the verification engine is able to come up with safe shields more quickly. Otherwise, \toolname{} spends the majority of its time (87\% on average) on shield synthesis.

\textbf{Safety.}
To validate whether the safety guarantee provided by \toolname{} is useful, we consider how DDPG and CPO behave during training. Specifically, Table~\ref{tab:violations} shows the average number of safety violations per run for DDPG and CPO. As we can see from this table, DDPG and CPO both exhibit safety violations in 8 out of the 10 benchmarks. In Figure~\ref{fig:safety_curves}, we show how the number of violations varies throughout the training process for a few of the benchmarks. The remaining plots are left to the supplementary material.
\begin{table}
    \centering
    \caption{Training time in seconds for network and shield updates.}
    \vspace{0.1in}
    {\small 
    \begin{tabular}{c|ccc}
         Benchmark & Network update (s) & Shield update (s) & Shield percentage \\ \hline
         mountain-car & 1900 & 5315 & 73.7\% \\
         road & 954 & 9401 & 90.8\% \\
         road-2d & 1015 & 19492 & 95.1\% \\
         noisy-road & 962 & 12793 & 93.0\% \\
         noisy-road-2d & 935 & 25514 & 96.5\% \\
         obstacle & 4332 & 27818 & 86.5\% \\
         obstacle2 & 4365 & 21661 & 83.2\% \\
         pendulum & 1292 & 113 & 8.0\% \\
         acc & 1097 & 56 & 4.9\% \\
         car-racing & 4361 & 15892 & 78.5\% \\
    \end{tabular}
    }
    \label{tab:runtime}
    \vspace{-0.2in}
\end{table}

\begin{figure}
\begin{subfigure}{0.33\textwidth}
    \centering
    \includegraphics[width=\textwidth]{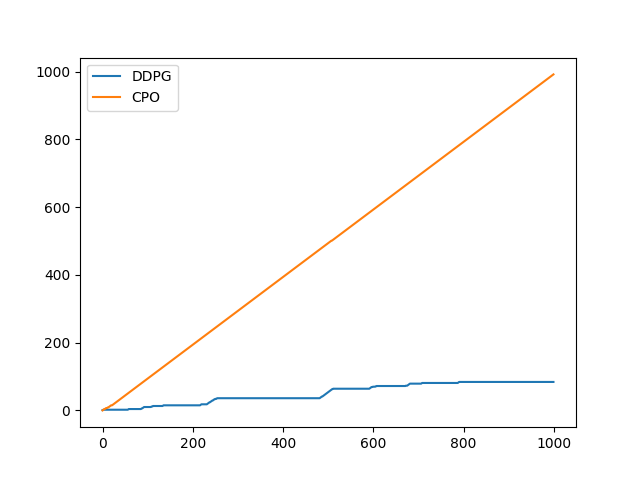}
    \caption{pendulum}
    \label{fig:pendulum_safety}
\end{subfigure}
\begin{subfigure}{0.33\textwidth}
    \centering
    \includegraphics[width=\textwidth]{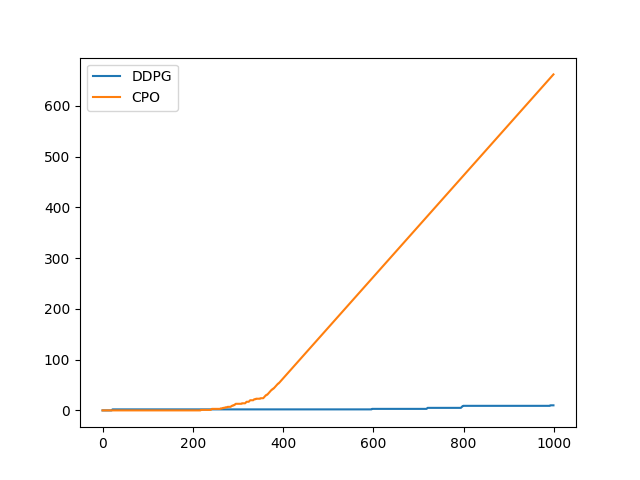}
    \caption{acc}
    \label{fig:acc_safety}
\end{subfigure}
\begin{subfigure}{0.33\textwidth}
    \centering
    \includegraphics[width=\textwidth]{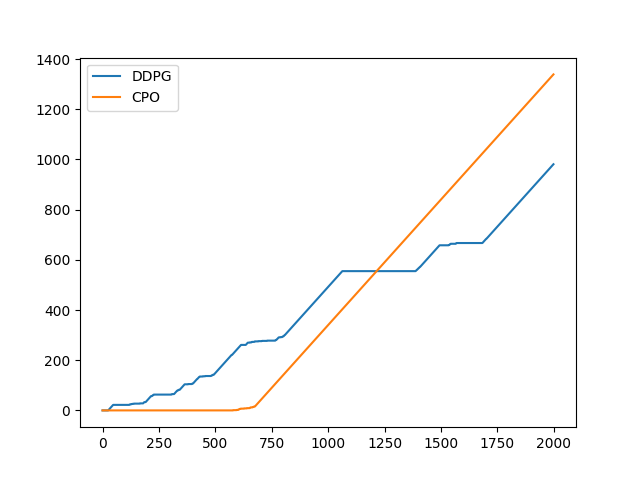}
    \caption{car-racing}
    \label{fig:car-racing_safety}
\end{subfigure}
\caption{Cumulative safety violations during training. }
\label{fig:safety_curves}
\end{figure}

\textbf{Qualitative Assessment.} To gain  intuition about the difference between policies that \toolname{} and the baselines compute, we consider trajectories from the trained policies for two of our benchmarks that are easy to visualize. 
Figure~\ref{fig:qualitative_mid_obstacle} shows the trajectories taken by each of the policies for the obstacle2 benchmark. In this environment, the policy starts in the lower left corner, and the goal is to move to the green circle in the upper right. However, the red box in the middle is unsafe.
As we can see from Figure~\ref{fig:qualitative_mid_obstacle}, all of the policies have learned to go around the unsafe region in the center. However DDPG has not reinforced this behavior enough and still enters the unsafe region at the corner. By contrast, the statically shielded policy manages to avoid the region, but there is a very clear bend in its trajectory where the shield has to step in. Revel avoids the unsafe region while maintaining a smooth trajectory throughout. In this case, CPO  also learns to avoid the unsafe region and go to the goal. (Because the environment is symmetrical, there is no significance to the CPO curve going up first and then right.)

\begin{wraptable}{l}{1.9in}
    \centering
    \caption{Safety violations.}
    {\small
    \begin{tabular}{l|cc}
      Benchmark & DDPG & CPO \\ \hline
      mountain-car & 0 & 3.6 \\
      road & 0 & 0 \\
      road-2d & 113.4 & 70.8 \\
      noisy-road & 1130.4 & 8526.4 \\
      noisy-road-2d & 107.4 & 0 \\
      obstacle & 12.4 & 1.0 \\
      obstacle2 & 96 & 118.6 \\
      pendulum & 92.4 & 9906 \\
      acc & 4 & 673 \\
      car-racing & 4956.2 & 22.4
    \end{tabular}
    }
    \label{tab:violations}
\end{wraptable}


Figure~\ref{fig:qualitative_acc} shows trajectories for ``acc'', which models an adaptive cruise control system where the goal is to follow a lead car as closely as possible without crashing into it. The lead car can apply an acceleration to itself at any time. The x-axis shows the distance to the lead car while the y-axis shows the relative velocities of the two cars. Here, all three trajectories start by accelerating to close the gap to the lead car before slowing down again. The statically shielded (and most conservative) policy is the first to slow down.
The \ddpg and \cpo policies fail to slow down soon enough or quickly enough and crash into the lead car (the red region on the right side of the figure). In contrast, the \toolname{} policy can more quickly close the gap to the lead car and slow down later while still avoiding a crash.

\begin{figure}
\begin{minipage}{0.5\textwidth}
    \centering
    \includegraphics[width=\textwidth]{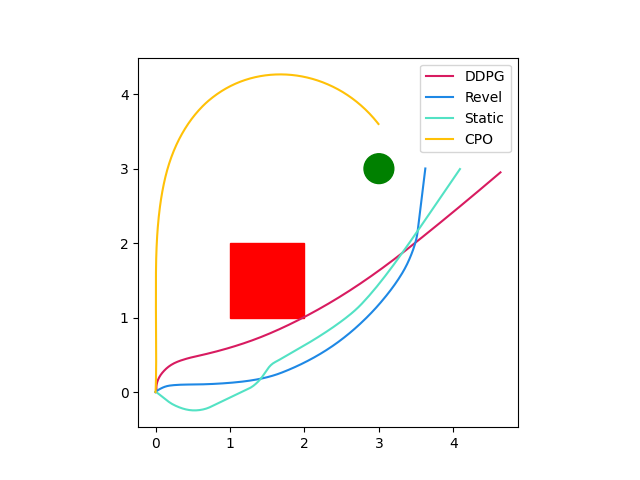}
    \caption{Trajectories for obstacle2.}
    \label{fig:qualitative_mid_obstacle}
\end{minipage}
\begin{minipage}{0.5\textwidth}
    \centering
    \includegraphics[width=\textwidth]{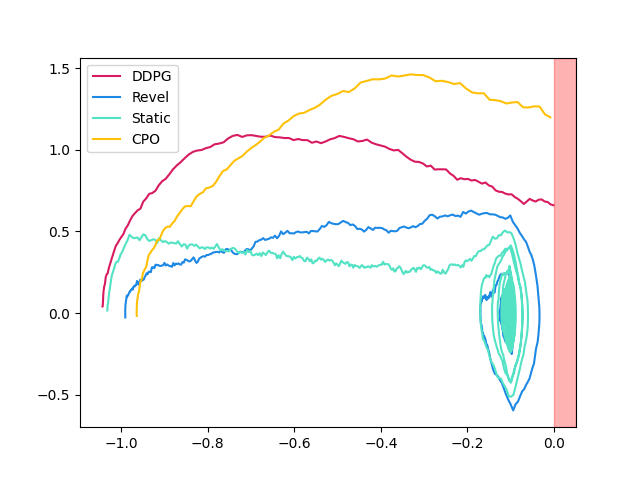}
    \caption{Trajectories for acc.}
    \label{fig:qualitative_acc}
\end{minipage}
\end{figure}




\section{Related Work}
\label{sec:Related}


There is a growing literature on safety in \rl~\cite{garcia_comprehensive_nodate}. Approaches here can be classified on basis of whether safety is guaranteed during learning or deployment. \toolname, and, for example, \cpo~\cite{cpo_main}, were designed to enforce safety during training. Another way to categorize approaches is by whether their guarantees are probabilistic (or in expectation) or worst-case. Most approaches~\cite{MoldovanA12,chow2018lyapunov,cpo_main,chow2018lyapunov} are in the former category; however, \toolname and prior approaches based on verified monitors~\cite{AlshiekhBEKNT18,fulton2018safe,fulton2019verifiably} are in the latter camp.
Now we discuss in more detail three especially related threads of work. 

\textbf{Safety via Shielding.}
These approaches rely on a definition of error states or fatal transitions to guarantee safety and have been used extensively in both RL and control theory \cite{ AkametaluKFZGT14, AlshiekhBEKNT18, ChowNDG18, fulton2018safe, fulton2019verifiably, GillulaT12, PerkinsB02, ZhuShielding}. Our approach follows this general framework, but crucially introduces a mechanism to improve the shielded-policy during training. This is achieved by projecting the neural policy onto the shielded policy space. The idea of synthesizing a shield to imitate a neural policy has been explored in recent work~\cite{viper, ZhuShielding}. However,  these approaches only generated the shield after training, so there are no guarantees about safety during training.


\textbf{Formal Verification of Neural Networks.} 
There is a growing literature on the verification of worst-case properties 
of neural networks~\cite{charon, ai2, reluplex, Marabou, reluval}.
In particular, a few recent papers~\cite{ivanov2019verisig,sun2019formal} target the verification of neural policies for autonomous agents.
However, performing such verification inside a learning loop is computationally infeasible -- in fact, state-of-the-art techniques failed to verify a single network from our benchmarks within half an hour. 


\textbf{Safety via Optimization Constraints.}
Many recent approaches to safe RL rely on specifying safety constraints as an additional cost function in the optimization objective~\cite{cpo_main, berkenkamp2017safe, dalal_safe_2018, le2019batch, li_temporal_2019}. These approaches typically provide safety up to a certain threshold by requiring that the additional cost function is kept below a certain constant. In contrast, our approach is suitable for use cases in which safety violations are completely unacceptable and where provable guarantees are required.


\section{Conclusion}
\label{sec:Conclusions}

We have presented \toolname, an \rl framework that permits formally verified exploration 
while supporting continuous action spaces and contemporary learning algorithms. Our key innovation is to cast the verified \rl problem as a form of mirror descent that uses a verification-friendly symbolic policy representation along with a neurosymbolic 
policy representation that benefits learning.

One limitation of this work is its assumption of a fixed worst-case model of the environment. Allowing this model to be updated as learning progresses~\cite{fulton2019verifiably} is a direction for future work. The development of incremental verification techniques to further allay the cost of repeated verification is another natural direction. Progress on such verification techniques can  potentially allow the use of more expressive classes of shields, which, in turn, can boost the learner's overall performance.

\newpage
\section*{Broader Impact}

In the recent past, reinforcement learning has seen numerous advances and found applications in safety-critical settings. System failures in this setting can result in significant loss of property or even loss of life. This work takes a step towards solving  this problem by guaranteeing that RL agents do not violate safety properties. 

As with any safety-related work, the consequences of failure or misuse of this technique can be severe. Specifically, there is a risk that a user might assume that their system is guaranteed safe when this is not the case (for example, if the user fails to adequately specify the environment or safety property). Writing correct safety specifications is known to be hard, so inexpert users may feel an unwarranted sense of security. While misuse of the tool carries great risk, proper use can confer substantial advantages. In particular, it may allow the benefits of RL to be brought to domains, such as robotics and autonomous vehicles, where failure has a very high cost.


\section*{Funding Acknowledgment} 

This work was supported in part by United States Air Force Contract \# FA8750-19-C-0092, ONR award \# N00014-20-1-2115, NSF Awards \# CCF-2033851, \# CCF-SHF-1712067, \# CCF-SHF-1901376, and \# CNS-CPS-1646522, and a JP Morgan Chase Fellowship (for Verma).

\begin{small}
\bibliographystyle{plain}
 \bibliography{main}
\end{small}

\newpage
\appendix

\section{Safely Imitating a Neural Policy}
In this section, we describe our projection algorithm for piecewise linear policies in more detail. Algorithm~\ref{alg:project} defines this operation at a high-level, but leaves out some of the details of the \textsc{ImitateSafely} subroutine. The role of \textsc{ImitateSafely} is to learn a linear policy which is provably safe on some region and behaves as similarly to the neural controller as possible on that region. Since linear policies are differentiable, we adopt a projected gradient descent approach. To formalize this, we note that a linear policy $g$ is just a matrix in $\mathbb{R}^{n \times m}$ where $n$ is the dimension of the action space and $m$ is the dimension of the state space. We will use $\theta_g$ to refer to a vector in $\mathbb{R}^{nm}$ parameterizing $g$.

\begin{algorithm}
\caption{Safely imitating a network using a given starting point and partition.}
\label{alg:imitate}
\begin{algorithmic}[1]
\STATE  {\bfseries Input:} A neural policy $f$, a region $\chi$, and a linear policy $g$.
   \STATE $g^* \gets g$
   \WHILE{*}
     \STATE $S \gets {\textsc{ComputeSafeRegion}}(g^*, \chi)$
     \STATE $g^* \gets g^* - \alpha \nabla D(g^*, N)$
     \STATE $g^* \gets \textit{proj}_S(g^*)$
   \ENDWHILE
 \STATE \textbf{return} $g^*$
 \end{algorithmic}
\end{algorithm}

Now our safe imitation algorithm is described in Algorithm~\ref{alg:imitate}. In each iteration, we first compute a safe region in the parameter space of $g^*$ over the region $\chi$. This is done by starting with a region bigger than the gradient step size and then iteratively searching for unsafe controllers and trimming the region to remove them. The returned region $S \subseteq \mathbb{R}^{n \times m}$ contains $g^*$ and only contains safe policies over the region $\chi$. This trimming process continues until $S$ can be verified using abstract interpretation~\cite{cousot1977}. In our implementation $S$ represents an interval in $\mathbb{R}^{nm}$ constraining $\theta_g$. Next, we take a gradient step according to the imitation loss $D$. For example $D$ may be computed using a DAgger-like algorithm to gather a dataset for supervised learning. Finally, we project $g^*$ into the safe region $S$ computed earlier. Specifically, this means projecting $\theta_{g^*}$ into the region of $\mathbb{R}^{nm}$ represented by $S$. Notice that since we project back into $S$ after each iteration, the policy returned by \textsc{ImitateSafely} is known to be safe on $\chi$.

Intuitively, recomputing $S$ at each iteration allows the controller to learn behavior which is more different from the starting point $g$ than would otherwise be possible. This is because computing a safe region involves abstracting the behavior of the system and in general it is intractable to compute the entire safe of safe policies in advance. Recomputing the safe space using the current controller at each step means we only need to prove the safety of a relatively small piece of the policy space local to the current controller. Specifically, if we can verify a region at least as large as one gradient step then the gradient descent procedure is unconstrained for that step. By repeating this process at each step, we only end up needing to verify a thin strip of policies surrounding the trajectory the gradient descent algorithm takes through the policy space.
\section{Theoretical Analysis}
Here we provide proofs of the theoretical results from Section~\ref{sec:Theory} and extend the discussion of a few theoretical issues.

Recall from Section~\ref{sec:Theory} that we require the policy space $\HH$ to be a vector space equipped with an inner product $\langle \cdot , \cdot \rangle$ inducing a norm $\|h\| = \sqrt{\langle h , h \rangle}$. Addition and scalar multiplication are defined in the standard way, i.e., $(\lambda u + \kappa v)(s) = \lambda u(s) + \kappa v(s)$. The cost functional of a policy $u$ is defined as $J(u) = \int_\mathcal{S} c(s, u(s)) \mathrm{d}\mu^u(s)$ where $\mu^u$ is the state distribution induced by $u$. We assume that $\G$ and $\FF$ are subspaces of $\HH$ so that there is a well-defined notion of distance between policies in these classes. Additionally, notice that while a policies in $\G$ may not be differentiable in terms of their programmatic representation, they may still be differentiable when viewed as points in the ambient space $\HH$. We will assume $\HH$ is parameterizable by a vector in $\mathbb{R}^N$ for some $N$.

We will make use of a few standard notions from functional analysis, restated here for convenience:



\begin{definition}
\label{def:strong-conv}
  (Strong convexity) A differentiable function $R$ is $\alpha$-strongly convex w.r.t.\ a norm $\|\cdot\|$ if $R(y) \ge R(x) + \langle \nabla R(x) , y - x \rangle + \frac{\alpha}{2} \|y - x\|^2$.
\end{definition}

\begin{definition}
\label{def:lipschitz-grad}
  (Lipschitz continuous gradient smoothness) A differentiable function $R$ is $L_R$-strongly smooth w.r.t.\ a norm $\|\cdot\|$ if $\|\nabla R(x) - \nabla R(y)\|_* \le L_R \|x - y\|$.
\end{definition}

\begin{definition}
\label{def:bregman}
  (Bregman divergence) For a strongly convex regularizer $R$, $D_R(x, y) = R(x) - R(y) - \langle \nabla R(y) , x - y \rangle$ is the Bregman divergence between $x$ and $y$. Note that $D_R$ is not necessarily symmetric.
\end{definition}



With these preliminaries, we can now prove Theorem~\ref{thm:regret} from Section~\ref{sec:Theory}. The high-level strategy for this proof will be to prove Lemma~\ref{lem:bias}, and then combine this result with a more general regret bound from~\cite{propel}. First we restate the general theorem below. Let $R$ be an $\alpha$-strongly convex and $L_R$-smooth functional w.r.t.\ the norm $\|\cdot\|$ on $\HH$. Additionally let $\nabla_\HH$ be a Fr\'{e}chet gradient on $\HH$. Then our algorithm can be described as follows: start with $g_0 \in \G$ (provided by the user) then for each iteration $t$:
\begin{enumerate}
    \item Compute a noisy estimate of the gradient $\widehat{\nabla} J(g_{t-1}) \approx \nabla J(g_{t-1})$.
    \item Update in $\HH$: $\nabla R(h_t) = \nabla R(h_{t-1}) - \eta \widehat{\nabla} J(g_{t-1})$.
    \item Perform an approximate projection $g_t = \proj_{\G}^R (h_t) \approx \argmin_{g \in \G} D_r(g, h_t)$.
\end{enumerate}

This procedure is approximate functional mirror descent under bandit feedback. We let $D$ be the diameter of $\G$, i.e., $D = \sup\{\|g - g'\| \mid g, g' \in \G\}$. $L_J$ is the Lipschitz constant of the functional $J$ on $\HH$. $\beta$ and $\sigma^2$ are bounds on the bias and variance, respectively,  of the gradient estimate in each iteration. $\alpha$ and $L_R$ are the strongly convex and smooth coefficients of the functional regularizer $R$. Finally $\epsilon$ is the bound on the projection error with respect to the same norm $\|\cdot\|$. We will make use of the following general result:

\begin{theorem} \cite{propel}
  Let $g_1, \ldots, g_T$ be a sequence of programmatic policies returned by $\toolname{}$ and $g^*$ be the optimal programmatic policy. We have the expected regret bound
  \[\mathbb{E}\left[\frac{1}{T} \sum_{t=1}^T J(g_t)\right] - J(g^*) \le \frac{L_R D^2}{\eta T} + \frac{\epsilon L_R D}{\eta} + \frac{\eta (\sigma^2 + L_J^2)}{\alpha} + \beta D.\]
  In particular, choosing $\eta = \sqrt{(1/T + \epsilon) / \sigma^2}$, this simplifies to
  \[\mathbb{E}\left[\frac{1}{T} \sum_{t = 1}^T J(g_t)\right] - J(g^*) = O\left(\sigma \sqrt{\frac{1}{T} + \epsilon} + \beta\right).\]
  \label{thm:propel-regret}
\end{theorem}

Now we restate and prove Lemma~\ref{lem:bias} from the main paper to provide a bound on the bias of our gradient estimate. Recall our definition of the immediate safety indicator $Z$ as zero if the shield is invoked and one otherwise. Recall the assumptions from Section~\ref{sec:Theory}:
\begin{enumerate}
    \item $J$ is convex in $\HH$ and 
    $\nabla J$ is $L_J$-Lipschitz continuous on $\HH$,
    \item $\HH$ is bounded (i.e., $\sup \{\|h - h'\| \mid h,h' \in \HH\} < \infty$),
    \item $\mathbb{E}[1- Z] \le \zeta$, i.e., the probability that the shield is invoked is bounded above by $\zeta$,
    \item the bias introduced in the sampling process is bounded by $\beta$, i.e., $\|\mathbb{E}[\widehat{\nabla}_\F \mid h] - \nabla_\FF J(h) \| \le \beta$, where 
    $\widehat{\nabla}_\F$ is the estimated gradient 
    \item for $s \in \states$, $a \in \actions$, and policy $h \in \HH$, if $h(a \mid s) > 0$ then $h(a \mid s) > \delta$ for some fixed $\delta > 0$.
\end{enumerate}
Under these assumptions:

\begin{lemma*}
  Let $D$ be the diameter of $\HH$, i.e., $D = \sup\{\|h - h'\| \mid h,h' \in \HH\}$. Then the bias incurred by approximating $\nabla_\HH J(h)$ with $\nabla_\FF J(h)$ and sampling is bounded by
  \[\left\|\mathbb{E}\left[\widehat{\nabla}_\FF \mid h\right] - \nabla_\HH J(h)\right\| = O(\beta + L_J \zeta).\]
\end{lemma*}
\begin{proof}
First, we note that $\|\mathbb{E}[\widehat{\nabla}_\FF \mid h] - \nabla_\HH J(h)\| \le \|\mathbb{E}[\widehat{\nabla}_\FF \mid h] - \nabla_\FF J(h)\| + \|\nabla_\FF J(h) - \nabla_\HH J(h)\|$. We have already assumed that the first term is bounded by $\beta$, so we will proceed to bound the second term.

Let $h = (g, \phi, f)$ be a policy in $\HH$. By the policy gradient theorem~\cite{sutton2000policy}, we have that
\begin{equation}
\nabla_\FF J(h) = \mathbb{E}_{s \sim \rho_h, a \sim h}\left[\nabla_\FF \log h(a \mid s) Q^h(s, a) \right]
\label{eqn:pg-thm}
\end{equation}
where $\rho_h$ is the state distribution induced by $h$ and $Q^h$ is the long-term expected reward from a state $s$ and action $a$. We will omit the distribution subscript in the remainder of the proof for convenience. Now note that if $Z$ is one, then then $h(a \mid s) = f(a \mid s)$, so that in particular
\[\nabla_\FF \log h(a \mid s) Q^h(s, a) = \nabla_\HH \log h(a \mid s) Q^h(s, a).\]
On the other hand, if $Z$ is zero, then $h(a \mid s)$ is independent of $f$, and so we have
\[\nabla_\FF \log h(a \mid s) Q^h(s, a) = 0.\]
Thus, we can rewrite Equation~\ref{eqn:pg-thm} as
\begin{align}
\nabla_\FF J(h) &= \mathbb{E} \left[ Z \nabla_\HH \log h(a \mid s) Q^h(s, a)\right] \notag \\
&= \mathbb{E}[Z] \mathbb{E}\left[\nabla_\HH \log h(a \mid s) Q^h(s, a) \right] + \text{Cov}(S, \nabla_\HH \log h(a \mid s) Q^h(s, a)) \notag \\
&= \mathbb{E}[Z] \nabla_\HH J(h) + \text{Cov}(S, \nabla_\HH \log h(a \mid s) Q^h(s, a)). \label{eqn:bias-lem-mid}
\end{align}

Note that the covariance term is a vector where the $i$'th component is the covariance between $Z$ and the $i$'th component of the gradient $\nabla_\HH^i$. Then for each $i$, by Cauchy-Schwarz we have
\[|\text{Cov}(Z, \nabla_\HH^i \log h(a \mid s) Q^h(s, a))| \le \sqrt{\text{Var}(Z) \text{Var}(\nabla_\HH^i \log h(a \mid s) Q^h(s, a))}.\]
Since $Z \in \{0, 1\}$ we must have $0 \le \text{Var}[Z] \le 1$ so that
\[|\text{Cov}(Z, \nabla_\HH^i \log h(a \mid s) Q^h(s, a))| \le \sqrt{\text{Var}(\nabla_\HH^i \log h(a \mid s) Q^h(s, a))}.\]
By assumption, for every state-action pair $(a, s)$ if $(a, s)$ is in the support of $\rho_h$ then $h(a \mid s) > \delta$. We also have that $Q^h(s, a)$ is bounded (because $J$ is Lipschitz on $\HH$ and $\HH$ is bounded). Then because the gradient of the log is bounded above by one and because $\nabla_HH$ is bounded by definition, we have $\|\nabla_\HH^i \log h(a \mid s) Q^h(s, a)\|$ is bounded. Therefore by Popoviciu's inequality, $\text{Var}(\nabla_\HH^i \log h(a \mid s) Q^h(s,a))$ is bounded as well. Choose $B > \text{Var}(\nabla_\HH^i \log h(a \mid s) Q^h(s,a))$ for all $i$. Then we have $\|\text{Var}(\nabla_\HH \log h(a \mid s) Q^h(s, a))\|_{\infty} < \sqrt{B}$, and because $\HH$ is finite-dimensional, $\|\text{Var}(\nabla_\HH \log h(a \mid s) Q^h(s, a))\| < c\sqrt{B}$ for some constant $c$ for any norm $\|\cdot\|$.

Substituting this into Equation~\ref{eqn:bias-lem-mid}, we have
\[\|\nabla_\FF J(h) - \mathbb{E}[S] \nabla_\HH J(h)\| < c\sqrt{B}.\]
Then
\begin{align*}
    \|\nabla_\FF J(h) - \nabla_\HH J(h)\| &\le \|\nabla_\FF J(h) - \mathbb{E}[S] \nabla_\HH J(h)\| + \|\mathbb{E}[S] \nabla_\HH J(h) - \nabla_\HH J(h)\| \\
    &< c\sqrt{B} + \|\mathbb{E}[S] \nabla_\HH J(h) - \nabla_\HH J(h)\|.
\end{align*}
By assumption, $\nabla_\HH J(h)$ is Lipschitz and $\HH$ is bounded. Let $D$ be the diameter of $\HH$ and recall that $L_J$ is the Lipschitz constant of $\nabla_\HH J(h)$. Choose an arbitrary $h_0 \in \HH$ and let $J_0 = \nabla_\HH J(h_0)$. Then for any policy $h \in \HH$ we have $\|\nabla_\HH J(h)\| \le J_0 + D L_J$. Then
\begin{align*}
\|\mathbb{E}[Z] \nabla_\HH J(h) - \nabla_\HH J(h)\| &= \|(\mathbb{E}[Z] - 1) \nabla_\HH J(h)\| \\
&= |\mathbb{E}[Z] - 1| \|\nabla_\HH J(h)\| \\
&\le |\mathbb{E}[Z] - 1| (J_0 + DL).
\end{align*}
Since $Z$ is an indicator variable, we have $0 \le \mathbb{E}[Z] \le 1$ so that $|\mathbb{E}[Z] - 1| = 1 - \mathbb{E}[Z]$. Then finally we assume $D$ is a known constant to simplify presentation, and arrive at
\[\|\nabla_\FF J(h) - \nabla_\HH J(h)\| < c \sqrt{B} + (1 - \mathbb{E}[Z])(J_0 + D L_J) = O(L_J \zeta)\]
and plugging this back into the original triangle inequality we have
\[\left\|\mathbb{E}\left[ \widehat{\nabla}_\FF \mid h \right] - \nabla_\HH J(h)\right\| = O(\beta + L_J \zeta).\]

\end{proof}
Now Theorem~\ref{thm:regret} follows directly by plugging this bound on gradient estimate bias into Theorem~\ref{thm:propel-regret}.

\section{Experimental Data and Additional Results}
In this section we provide more details about our experiments along with additional results.

First, we give a qualitative description of each benchmark:
\begin{itemize}
    \item mountain-car is a continuous version of the classic mountain car problem. In this environment the goal is to move an underpowered vehicle up a hill by rocking back and forth in a valley to build up momentum. The safety property asserts that the car does not go over the crest of the hill on the left.
    \item road, road-2d, noisy-road, and noisy-road-2d are all variants of an autonomous car control problem. In each case, the car's goal is to move to a specified end position while obeying a given speed limit. The noisy variants introduce noise in the environment, while the 2d variants involve moving in two dimensions to reach the goal.
    \item In obstacle and obstacle2, a robot moving in 2D space must reach a goal position while avoiding an obstacle. In obstacle this obstruction is placed off to the side so it only affects the agent during exploration (but the shortest path to the goal does not intersect it). In obstacle2, the obstruction is placed between the starting position and the goal so that the policy must learn to move around it (see Figure~\ref{fig:qualitative_mid_obstacle}).
    \item pendulum is a classic pendulum environment where the system must swing a pendulum up until it is vertical. The safety property in this case is a bound on the angular velocity of the pendulum.
    \item acc is an adaptive cruise control benchmark taken from~\cite{fulton2018safe} and modified to use a continuous action space. Here the goal is to follow a lead car as closely as possible without crashing into it. At each time step the lead car chooses an acceleration at random (from a truncated normal distribution) to apply to itself.
    \item car-racing is similar to obstacle2 except that in this case the goal is to reach a goal state on the opposite side of the obstacle and then come back. This requires the agent to complete a loop around the obstacle.
\end{itemize}
For each benchmark, we consider a bounded-time variant of the desired safety property. That is, for some fixed $T$ we guarantee that a policy $h = (g, \phi, f)$ cannot violate the safety property within $T$ time steps starting from any state satisfying $\phi$.

For most benchmarks, we train for 100,000 environment interactions with a maximum episode length of 100. For mountain-car we use a maximum episode length of 200 and 200,000 total environment interactions. For obstacle, obstacle2, and car-racing we use an episode length of 200 with 400,000 total environment interactions. For every benchmark we synthesize five new shields at even intervals throughout training. To evaluate CPO we use the implementation provided with the Safety Gym repository~\cite{Ray2019}. To account for our safety critical benchmarks, we reduce the tolerance for safety violations in this implementation by lowering the corresponding hyper-parameter from $25$ to $1$. For DDPG, we use an implementation from prior work~\cite{ZhuShielding}, which is also what we base the code for \toolname{} on. We ran each experiment with five independent, randomly chosen seeds. Note that the chosen number of training episodes was enough for the baselines to converge, in the sense that over the last 25 training episodes we see less than a 2\% improvement in policy performance.



We now provide more details about the safety violations seen during training. The plots in Figure~\ref{fig:safety-curves-extra} show the number of safety violations over time for DDPG and CPO. This figure is the same as Figure~\ref{fig:safety_curves} except that it shows information for every benchmark.
\begin{figure}
\begin{subfigure}{0.33\textwidth}
    \centering
    \includegraphics[width=0.95\textwidth]{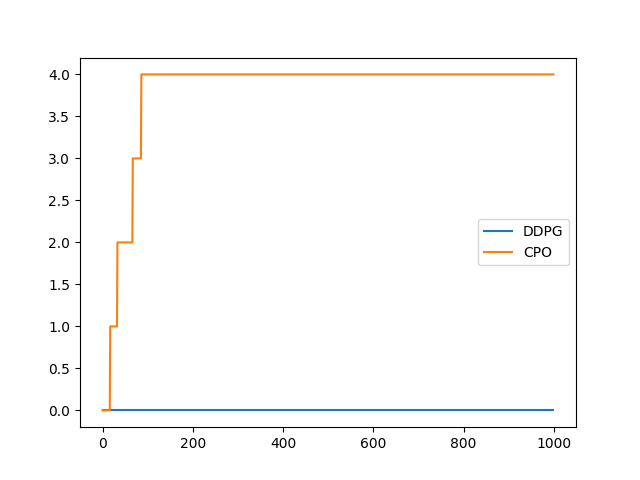}
    \caption{mountain-car}
    \label{fig:mountain_car_training_safety}
\end{subfigure}
\begin{subfigure}{0.33\textwidth}
    \centering
    \includegraphics[width=0.95\textwidth]{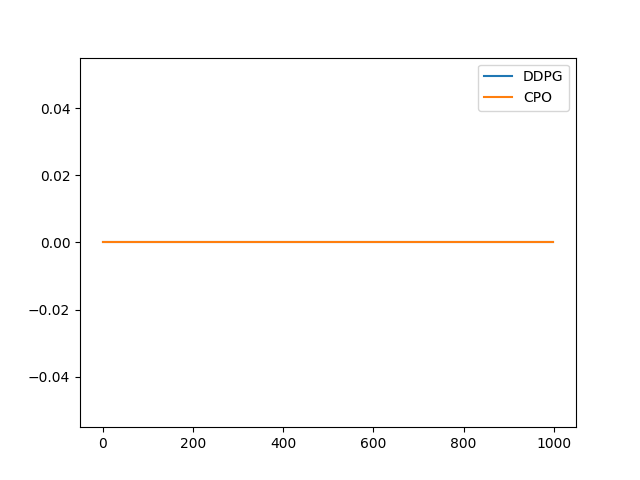}
    \caption{road}
    \label{fig:road_training_safety}
\end{subfigure}
\begin{subfigure}{0.33\textwidth}
    \centering
    \includegraphics[width=0.95\textwidth]{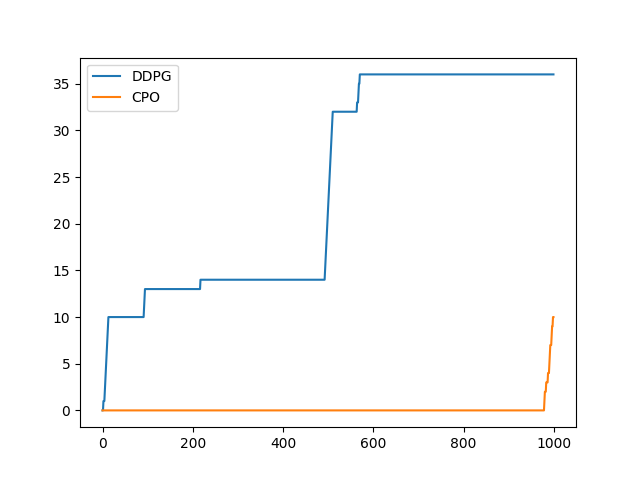}
    \caption{road-2d}
    \label{fig:road_2d_training_safety}
\end{subfigure}
\begin{subfigure}{0.33\textwidth}
    \centering
    \includegraphics[width=0.95\textwidth]{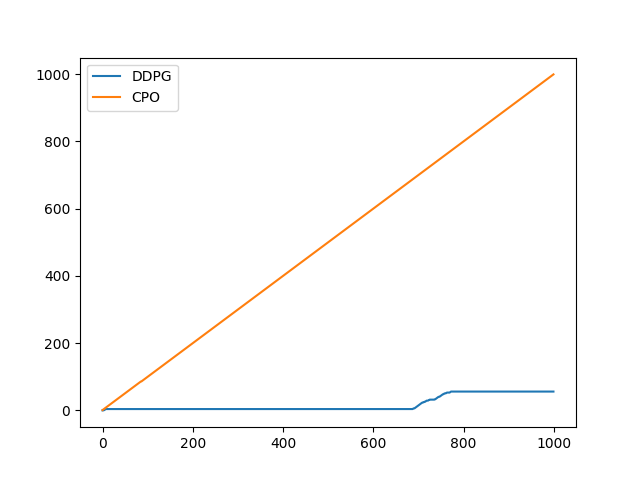}
    \caption{noisy-road}
    \label{fig:noisy_road_training_safetu}
\end{subfigure}
\begin{subfigure}{0.33\textwidth}
    \centering
    \includegraphics[width=0.95\textwidth]{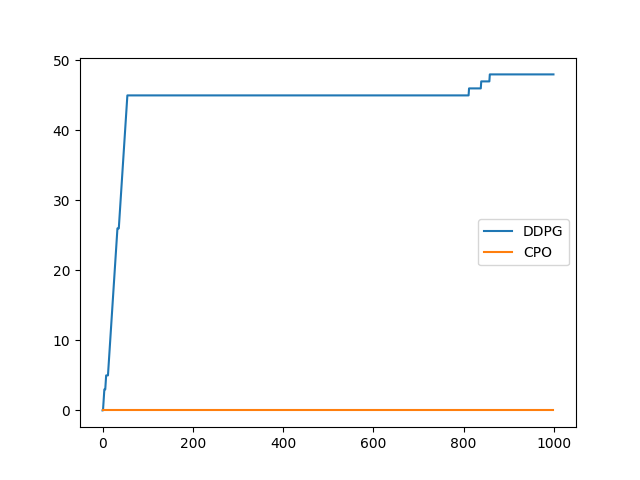}
    \caption{noisy-road-2d}
    \label{fig:noisy_road_2d_training_safety}
\end{subfigure}
\begin{subfigure}{0.33\textwidth}
    \centering
    \includegraphics[width=0.95\textwidth]{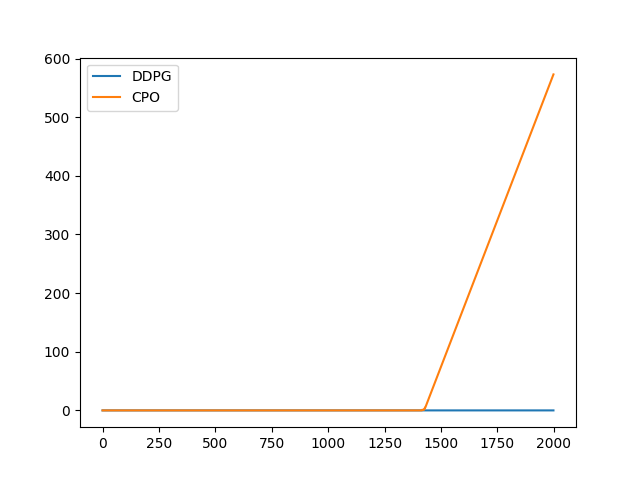}
    \caption{obstacle}
    \label{fig:obstacle_training_safety}
\end{subfigure}
\begin{subfigure}{0.33\textwidth}
    \centering
    \includegraphics[width=0.95\textwidth]{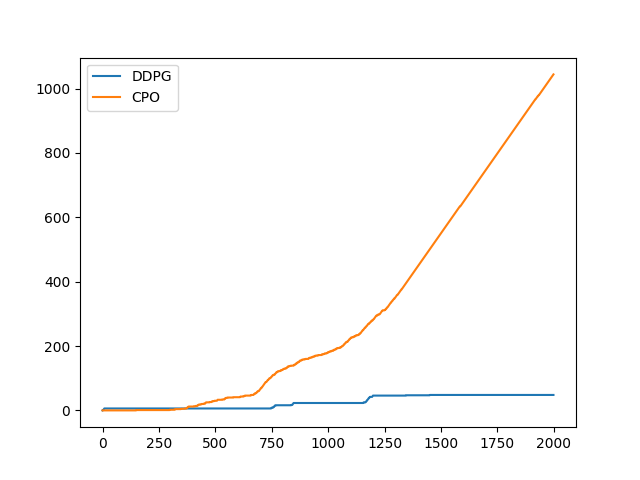}
    \caption{obstacle2}
    \label{fig:mid_obstacle_training_safety}
\end{subfigure}
\begin{subfigure}{0.33\textwidth}
    \centering
    \includegraphics[width=0.95\textwidth]{pendulum_safety_new.png}
    \caption{pendulum}
    \label{fig:pendulum_training_safety}
\end{subfigure}
\begin{subfigure}{0.33\textwidth}
    \centering
    \includegraphics[width=0.95\textwidth]{acc_safety_new.png}
    \caption{acc}
    \label{fig:acc_training_safety}
\end{subfigure}
\begin{subfigure}{0.33\textwidth}
    \centering
    \includegraphics[width=0.95\textwidth]{car-racing_safety_new.png}
    \caption{car-racing}
    \label{fig:car_racing_training_safety}
\end{subfigure}
\caption{Safety violations over time.}
\label{fig:safety-curves-extra}
\end{figure}

\end{document}